\documentclass{article}

\usepackage[preprint]{neurips_2023}

\usepackage[utf8]{inputenc} % allow utf-8 input
\usepackage[T1]{fontenc}    % use 8-bit T1 fonts
\usepackage{hyperref}       % hyperlinks
\usepackage{url}            % simple URL typesetting
\usepackage{booktabs}       % professional-quality tables
\usepackage{amsfonts}       % blackboard math symbols
\usepackage{nicefrac}       % compact symbols for 1/2, etc.
\usepackage{microtype}      % microtypography
\usepackage{xcolor}         % colors
\usepackage{amsmath}
\usepackage{amssymb}
\usepackage{mathtools}
\usepackage{amsthm}
\usepackage{verbatim}
\usepackage[capitalize,noabbrev]{cleveref}

\theoremstyle{plain}
\newtheorem{theorem}{Theorem}[section]

\newtheorem{lemma}[theorem]{Lemma}
\newtheorem{corollary}[theorem]{Corollary}
\theoremstyle{definition}
\newtheorem{definition}[theorem]{Definition}

\theoremstyle{remark}

\newcommand{\RR}{{\mathbb R}}

\newcommand{\EE}{\mathbb{E}}

\DeclareMathAlphabet{\altmathcal}{OMS}{cmsy}{m}{n}

\renewcommand{\cite}[1]{\citep{#1}}  % use natbib's citep for cite, same as ICML template
\setcitestyle{authoryear,round,citesep={;},aysep={,},yysep={;}}

\title{DP-TBART: A Transformer-based Autoregressive Model for Differentially Private Tabular Data Generation}

\author{%
  Rodrigo Castellon\thanks{Work done as an intern at Bloomberg.} \\
  Department of Computer Science\\
  Stanford University\\
  Stanford, CA USA \\
  \texttt{rjcaste@stanford.edu} \\
  \And
  Achintya Gopal \\
  Bloomberg \\
  New York, NY USA \\
  \texttt{agopal6@bloomberg.net} \\
  \AND
  Brian Bloniarz \\
  Bloomberg \\
  San Francisco, CA USA \\
  \texttt{bbloniarz@bloomberg.net} \\
  \And
  David Rosenberg \\
  Bloomberg \\
  Toronto, ON Canada \\
  \texttt{drosenberg44@bloomberg.net} \\
}

\begin{document}

\maketitle

\begin{abstract}
The generation of synthetic tabular data that preserves differential privacy is a problem of growing importance.
While traditional marginal-based methods have achieved impressive results, recent work has shown that deep learning-based approaches tend to lag behind.
In this work, we present Differentially-Private TaBular AutoRegressive Transformer (DP-TBART), a transformer-based autoregressive model that maintains differential privacy and achieves performance competitive with marginal-based methods on a wide variety of datasets, capable of even outperforming state-of-the-art methods in certain settings.
We also provide a theoretical framework for understanding the limitations of marginal-based approaches and where deep learning-based approaches stand to contribute most.
These results suggest that deep learning-based techniques should be considered as a viable alternative to marginal-based methods in the generation of differentially private synthetic tabular data.
\end{abstract}

\section{Introduction}

Though sharing tabular data about individuals can be broadly beneficial (for medical research, for example), privacy concerns typically prevent such data sharing.
Differentially private synthetic data generation stands out as an appealing solution to this problem: it provides strong formal privacy guarantees and avoids having to anticipate the exact analysis a downstream analyst might want to conduct, while producing a synthetic dataset that ``looks like'' the real data from the perspective of an analyst.
This problem has received considerable attention from the research community, with a wide variety of approaches available in the literature.
Among these, approaches tend to fall into two categories:

\begin{itemize}
    \item \textit{Marginal-based approaches} make differentially private measurements of low-order marginals, typically represented as histograms, and then fit a generative model to these histograms.
    \item \textit{Deep learning-based approaches} are inspired by successes in domains such as images and text~\cite{li2021large} and train deep neural networks using differentially private optimizers such as DP-SGD~\cite{abadi2016deep} to directly fit the data.\footnote{In general, we follow the nomenclature set in~\citet{tao2021benchmarking}.
    For deep learning-based approaches, we deviate because of the recent development of non-GAN approaches.}
\end{itemize}

Recent works unambiguously establish marginal-based approaches as state-of-the-art, with performance far surpassing that of approaches using neural networks, which typically fail to adequately model even one-dimensional marginals~\cite{tao2021benchmarking,bowen2019comparative}.

Given the lack of success of deep learning relative to marginal-based approaches, we ask the following question: is this significant gap due to a fundamental inadequacy in the deep learning approach to the task? If not, how can we close this gap?

To answer this question, we present Differentially-Private TaBular AutoRegressive Transformer (DP-TBART), a transformer-based autoregressive model explicitly tailored for the tabular domain, and we show that this simple approach significantly bridges the gap on 10 separate datasets.
In addition to narrowing this gap, we provide preliminary results suggesting that this modeling approach serves as a promising direction for downstream use cases that require accurately modeling higher-order interactions among columns.

\section{Background \& Related Work}

\subsection{Differential Privacy}

Differential privacy (DP), first introduced by~\citet{dwork2006calibrating}, is a technique used to protect the privacy of individuals in data analysis by preventing attackers from inferring specific information from the data.
While many formulations of differential privacy exist, we choose to follow the standard formulation parameterized by $\epsilon$ and $\delta$, with lower values providing stronger privacy guarantees.
More formally, a differentially-private algorithm is defined as follows:

\begin{definition}
Let $\altmathcal{A} : \altmathcal{D} \rightarrow \altmathcal{S}$ be a randomized algorithm, and let $\epsilon > 0, \delta \in [0,1]$. We say that $\altmathcal{A}$ is $(\epsilon, \delta)$-differentially private if for any two neighboring datasets $D, D' \in \altmathcal{D}$ differing by a single record, we have that
$$
P[\altmathcal{A}(D)\in S]\leq e^\epsilon P[\altmathcal{A}(D')\in S] + \delta
$$
for all $S\subset\altmathcal{S}$.
\end{definition}

In this paper, we consider a ``record'' to be a row in a tabular dataset, but this varies depending on the use case.
Furthermore, we take ``differing by a single record'' to mean the addition/removal of an arbitrary record from the dataset (this is known as ``unbounded DP''), following the convention set in previous works~\cite{li2021large,mckenna2022aim}.

More recently, methods for training neural networks in a differentially private manner have been proposed~\cite{abadi2016deep,papernot2016semi}.
In this work, we focus on DP-SGD, a method that preserves privacy throughout training by clipping \& perturbing gradients.
The per-sample gradients are first projected down to a radius $C$ ball (with $C$ designated as a hyperparameter) in order to bound their sensitivities, and isotropic Gaussian noise is further added to ensure the privacy guarantee is satisfied.
The final gradient is obtained by taking the mean of the noised per-sample gradients.
Formally, the differentially private estimate of the gradient from a batch of $n$ examples with per-sample gradients $g_i$ is computed as

$$
g\coloneqq \frac1n \sum_{i=1}^n \left[\min\left(\frac{C}{||g_i||}, 1\right) \cdot g_i + \altmathcal{N}(0, C^2 \sigma^2 I) \right]
$$

where $\sigma$ is pre-computed by performing a binary search for the smallest suitable $\sigma$ that satisfies a desired $(\epsilon,\delta)$-DP guarantee for a training run with a set batch size and number of training steps.

\subsection{Tabular data generation with deep learning}

Many methods for generating tabular data with deep learning have been proposed.
Among these, some of the most popular approaches include GANs~\cite{xu2019modeling,zhao2021ctab,zhao2022ctab} and, more recently, autoregressive models~\cite{borisov2022language,canale2022generative}, diffusion models~\cite{kotelnikov2022tabddpm}, normalizing flows~\cite{lee2022differentially}, VAEs~\cite{xu2019modeling}, and MMD~\cite{harder2021dp,yang2023differentially}.

Unfortunately, while some of these works train with differential privacy, they often neglect to mention the state-of-the-art marginal-based approaches, giving the false impression that these deep learning approaches are state-of-the-art on the task.
In contrast, our paper explicitly compares these two approaches on a wide variety of datasets, shows that a performance gap between more recent deep learning methods and marginal-based methods still exists, and demonstrates that autoregressive models can bridge the gap.

Among the methods listed above, \cite{canale2022generative} is the most similar to our method, which presents a universal attention-based modeling strategy for generating synthetic hierarchical data.
However, despite the importance of differentially private tabular data generation, the paper presents only a single experiment that validates that their model can achieve competitive performance on this task on a single metric.
We expand on these preliminary results and demonstrate that a similar approach is applicable to a wider variety of datasets, and in some limited cases even outperforms state-of-the-art marginal-based methods.
Furthermore, we bolster our results with a theoretical framework that sheds light into the specific limitations of marginal-based methods, and \textit{why} deep learning-based methods can sometimes outperform them.
As such, this paper stands to be the first to provide a comprehensive comparison between deep learning-based methods and marginal-based methods on the task of differentially private tabular data generation.

\subsection{Tabular data generation with marginal-based approaches}

In general, marginal-based approaches make differentially-private measurements (e.g., column histograms)
and then fit these measurements with a model; both steps may each be performed exactly once~\cite{mckenna2019graphical} or, more often, interleaved many times~\cite{zhang2017privbayes,liu2021iterative,mckenna2022aim}.

Methods that fall under this approach distinguish themselves by (1) which measurements they make and (2) the model used.
For example, \citet{mckenna2019graphical} proposes using probabilistic graphical models (PGMs) to represent the distribution.
Other methods use Bayesian networks~\cite{zhang2017privbayes} or neural networks~\cite{liu2021iterative}.

The inherent weakness of these methods is that they access all information about the data distribution in the form of low-order marginals,
due to the computational complexity of the exponential scaling to high-order marginals.
This limitation forces marginal-based methods to prioritize simple correlations at the expense of more complex interactions among columns.
We make this intuition more concrete in~\cref{subsubsec:theoretical}.

\section{Method}

\subsection{Preprocessing}

Consistent pre-processing and assumptions about what is considered public or private information is crucial for fair comparisons between methods.
In this work, we include datasets with both numeric and categorical column types, and consider the column type (categorical or numerical) as well as their ``min'' and ``max'' values to be public (in other words, no DP budget needs to be used to compute max and min).

For the methods ``AIM''~\cite{mckenna2022aim} and ``DP-TBART'', we uniformly discretize the numeric columns into 100 bins using the publicly available min and max values.
Then, at inference time we invert the discrete values to the midpoint of their original bin edges, rounding to integer values if necessary.
For the method ``DPCTGAN''~\cite{rosenblatt2020differentially}, values are normalized based on the original min and max.

\subsection{Autoregressive Modeling}
\label{subsec:autoregressive}

We adopt a language modeling procedure (akin to many natural language modeling procedures~\cite{vaswani2017attention,radford2019language}) specifically tailored for the discrete tabular setting.
Let $K$ be the number of columns in the given dataset.
We factorize the probability distribution over a row $X = (X_1,X_2,\ldots,X_K)$ with the chain rule:

$$
P(X) = P(X_1) \prod_{k=2}^K P(X_k | X_{<k})
$$
Note that here, values in each column encode to tokens in a shared vocabulary $V$ while ensuring that no two columns share an encoded value.\footnote{Though we maintain this design choice for all datasets included in~\cref{tab:main_table}, we share tokens across columns for the Dyck and character-level datasets.}
As a consequence, we have that $|V| = \sum_i |V_i|$, where $V_i$ represents the token encodings for column $i$.
The model is defined as a function $f_\theta$ that takes a sequence in $V^k$ (with $0\leq k< K$) and outputs a vector of logits in $\RR^{|V|}$.
To force the model $f_\theta$ to represent a probability distribution over only valid tokens $P(X)$ and thus to output only valid samples, we post-process the model output at each column $i$ and mask out all elements irrelevant to column $i$.
That is, for some sequence $x$ of length $k$, our revised logit outputs are

$$
f'_\theta(x) \coloneqq f_\theta(x) \hspace{1pt} \odot \hspace{1pt} \mathbf{m}
$$

where $\odot$ represents the element-wise (Hadamard) product and $\mathbf{m}$ sets all elements between index $\sum_{i=1}^k |V_i|$ and $\sum_{i=1}^{k+1} |V_i|$ to one and $-\infty$ everywhere else.

We sample from trained models directly, without any further post-processing on the distribution.

\subsection{Implementation Details}

We adopt a three-layer transformer decoder-only architecture~\cite{vaswani2017attention} closely resembling DistilGPT-2 (hidden dimension $768$, $12$ attention heads) with learned positional encodings and train our models with DP-Adam, learning rate $3e-4$, and batch size $256$ for all of our experiments unless otherwise indicated.
During training, we randomly set aside 1\% of the data for validation and treat the rest as the training set.
For more information, see the Appendix.

\section{Experiments}

\subsection{Experimental Setup}

\subsubsection{Datasets}
We use 10 datasets from various sources for our experiments.
Please note that the columns designated as the ``Target Variable'' are \textit{only} used for evaluation purposes and are not given any special treatment during training.
In other words, the fact that a column is a target variable does not influence the training process.

\begin{table*}[t]
  \centering
  \begin{tabular}{lccccc}
    \toprule
    Dataset & \# of Rows & Categorical & Numeric & Problem & Target Variable \\
    \midrule
    Adult & 48,842 & 9 & 5 & Classification & income\textgreater50K \\ 
    King & 21,613 & 7 & 13 & Regression & price \\
    Insurance & 1,338 & 4 & 3 & Regression & charges \\
    Loan & 5,000 & 7 & 7 & Classification & Personal Loan \\ 
    Credit & 284,807 & 1 & 30 & Classification & Class \\ 
    Bank & 45,211 & 10 & 7 & Classification & y \\ 
    Census & 299,285 & 35 & 6 & Classification & income \\ 
    Car & 1,728 & 7 & 0 & Classification & class \\ 
    Mushroom & 8,124 & 23 & 0 & Classification & poisonous \\ 
    Poker Hands & 25,010 & 11 & 0 & Classification & 10 \\ 
    \bottomrule
\end{tabular}
    \caption{Summary of datasets used in our experiments.}
    \label{tab:dataset_table}
  \end{table*}

\subsubsection{Methods}

In addition to DP-TBART, we evaluate the following methods: \textbf{AIM}~\cite{mckenna2022aim}, the most recent marginal-based method; \textbf{DP-NTK}~\cite{yang2023differentially}, the most recent deep learning-based method; and \textbf{DP-MERF}~\cite{harder2021dp} and \textbf{DP-HFlow}~\cite{lee2022differentially} and \textbf{DPCTGAN}~\cite{rosenblatt2020differentially}, additional deep learning-based methods.

\subsubsection{Evaluation}

We evaluate using the following metrics, which are adapted from the Synthetic Data Vault (SDV)~\cite{7796926}.

\begin{itemize}
    \item \href{https://docs.sdv.dev/sdmetrics/metrics/metrics-glossary/kscomplement}{\textbf{Kolmogorov-Smirnov Test (KS)}}: For numeric columns, this metric computes the Kolmogorov-Smirnov statistic $KS$ (the maximum difference between the CDF of the synthetic and real columns) and returns $1 - KS$.
    This metric lies between $0$ and $1$, and higher is better.
    We average this value across all numeric columns.
    \item \href{https://docs.sdv.dev/sdmetrics/metrics/metrics-in-beta/cstest}{\textbf{Chi-Square Test (CS)}}: For categorical columns, this metric returns the $p$-value from a chi-square test that checks the null hypothesis that the synthetic data comes from the same distribution as the real data.
    We average this value across all columns.
    This metric lies between $0$ and $1$, and higher is better.
    \item \href{https://docs.sdv.dev/sdmetrics/metrics/metrics-in-beta/detection-single-table}{\textbf{Detection (DET)}}: We train an XGBoost model to differentiate between real and synthetic samples, and then report $1 - ROCAUC$, where $ROCAUC$ is computed as an average of the ROC-AUC scores across several folds.
    This metric lies between $0$ and $1$, and higher is better (we would like a classifier to \textit{not} be able to distinguish synthetic data from real data).
    \item \href{https://docs.sdv.dev/sdmetrics/metrics/metrics-in-beta/ml-efficacy-single-table/regression}{\textbf{ML Efficacy (Regression)}}: We fit a linear regression model to predict a target variable in the dataset on synthetic data and evaluate its performance on the real dataset.
    This metric returns the $r^2$ test score.
    This metric lies below $1$, and higher is better.
    \item \href{https://docs.sdv.dev/sdmetrics/metrics/metrics-in-beta/ml-efficacy-single-table/binary-classification}{\textbf{ML Efficacy (Classification)}}: We fit a decision tree to predict a target variable in the dataset on synthetic data and evaluate its performance on the real dataset.
    This metric returns the F1 test score.
\end{itemize}

Generally, we interpret the ``KS'' and ``CS'' metrics to capture how well a model is able to replicate low-order marginals and the ``DET'' and ``ML'' metrics to capture the model's ability to capture higher-order, more complex correlations.
We follow this rough intuition in our later analysis.

\subsection{Compute}

Our experiments were conducted on a DGX box with 8 NVIDIA Tesla V100 GPUs, with each individual model run executed on a single GPU.
Each model required roughly 20 minutes to 3 hours to train, depending on the input dataset size.
This resulted in a total compute usage of approximately 200 GPU hours for our main experiments.

\subsection{Benchmark}

\begin{table*}[t]
  \centering
  \small
  \begin{tabular}{lccccc}
    \toprule
    Method & KS $\uparrow$ & CS $\uparrow$ & DET $\uparrow$ & ML (CLF) $\uparrow$ & ML (REG) $\uparrow$ \\
    \midrule
    AIM & 0.886 $\pm$ 0.001 & 0.997 $\pm$ 0.001 & 0.260 $\pm$ 0.002 & 0.594 $\pm$ 0.002 & 0.132 $\pm$ 0.057 \\
    \midrule
    DP-TBART & \textbf{0.847 $\pm$ 0.001} & \textbf{0.978 $\pm$ 0.002} & \textbf{0.172 $\pm$ 0.004} & \textbf{0.554 $\pm$ 0.005} & \textbf{-0.264 $\pm$ 0.022} \\
    DPCTGAN & 0.591 $\pm$ 0.006 & 0.831 $\pm$ 0.01 & 0.031 $\pm$ 0.001 & 0.412 $\pm$ 0.012 & -2.83e4 $\pm$ 2.31e4 \\
    DP-MERF & 0.505 $\pm$ 0.006 & 0.532 $\pm$ 0.003 & 0.0 $\pm$ 0.0 & 0.491 $\pm$ 0.014 & -4.2e20 $\pm$ 2.5e20  \\
    DP-HFlow* & 0.676 $\pm$ 0.004 & 0.622 $\pm$ 0.008 & 0.034 $\pm$ 0.004 & 0.543 $\pm$ 0.019 & -0.603 $\pm$ 0.345 \\
    DP-NTK & 0.349 $\pm$ 0.001 & 0.438 $\pm$ 0.006 & 0.0 $\pm$ 0.0 & 0.502 $\pm$ 0.011 & -5.83e16 $\pm$ 4.76e16 \\
    \bottomrule
  \end{tabular}
    \caption{We evaluate AIM, DP-TBART, and four deep learning baselines against a variety of metrics, and report the average scores (and their standard errors) across all datasets. 
    Note that ML (REG) may take on negative values if the regressor's predictions perform worse on the real test data than a baseline that knows only the true mean.
    (*) For DP-HFlow, we re-implement this method as no code is publicly available.
    }
    \label{tab:main_table}
\end{table*}

As shown in~\cref{tab:main_table}, we substantially outperform all three deep learning baselines included in our study, reaching performance comparable with the leading marginal-based approach, AIM~\cite{mckenna2022aim}.

We use unbounded DP (following~\cite{li2021large,mckenna2022aim}) for all approaches, with a choice of $\epsilon=1$ and $\delta=1e-9$ for all datasets and approaches.

\subsection{Limitations of Marginal-based Approaches}

Though AIM outperforms all other methods on our benchmark (see~\cref{tab:main_table}), a key limitation is that since it measures only low-order marginals, it cannot access the complex higher-order correlations that deep learning-based methods can.
In this section, we will first propose a theoretical framework that precisely characterizes this limitation and then use this framework to find two real-world settings where marginal-based approaches fail to achieve good performance.

\subsubsection{Theoretical Framework}
\label{subsubsec:theoretical}

We provide a theoretical analysis of marginal-based methods, which currently perform well on standard benchmarks, but are limited by construction in their ability to encode joint distributions with multiple interactions between variables
Our information theoretical framework demonstrates this limitation and gives a closed-form expression for the model's divergence in the ideal setting of infinite data and privacy budget.

\paragraph{Preliminaries}
\label{par:preliminaries}

To begin, assume a data-generating process yields a discrete distribution $\bar{P}(X)$ defined on a sample space $\altmathcal{X} = (V_1,\ldots,V_K)$ with $V_i$ disjoint w.r.t. each other as defined in~\cref{subsec:autoregressive}.
We refer to $\bar{P}(X)$ as the \textit{true distribution}.

Next, a two-player adversarial game containing two agents, modeler and nature, ensues.
The modeler is given access to the set of all marginals of order $M$, i.e., $\bar{P}(X_{i_1},\ldots,X_{i_M})$ for all possible subsets $\{i_1,\ldots,i_M\} \subseteq \{1,2,\ldots,K\}$, but nothing else.
This is consistent with a setting in which there is infinite data and an infinite privacy budget.

The modeler must propose a distribution $Q_M(X_1,X_2,\ldots,X_K)$ such that all marginals match, i.e., $Q_M(X_{i_1},\ldots,X_{i_M}) = \bar{P}(X_{i_1},\ldots,X_{i_M})$ for all subsets $\{i_1,\ldots,i_M\}$.
We call the set of all distributions such that all marginals match $\Pi_M$.
Then, nature reacts to $Q_M$ and chooses $P_M\in\Pi_M$.

The modeler is asked to minimize the log loss $L(P_M, Q_M) = \EE_{P_M}[-\log Q_M(x)]$, and nature maximizes the same quantity.
We further assume that when nature encounters ties along level sets of $f(P_M) = L(P_M, Q_M)$, they are broken by deferring to maximizing the quantity $KL(Q_M || P_M)$ (but that any choice based on a further tie is arbitrary).

\paragraph{Results}

We first appeal to a set of results from~\citet{grunwald2004game}, which we restate as a single lemma using our notation.

\begin{lemma}
\label{lemma:1}
\cite{grunwald2004game}

\begin{enumerate}
    \item The modeler's optimal estimate is uniquely given by the maximum entropy distribution $Q_M^* = \arg\sup_{P\in\Pi_M} H(P)$, where $H$ is Shannon entropy.
    \item $L(P,Q_M^*) = L(P',Q_M^*) = H(Q_M^*)$ for all $P,P'\in\Pi_M$.
\end{enumerate}
\end{lemma}

These initial results provide a couple of key insights.
First, if a modeler acts optimally (in the sense that they pick the $Q_M$ that, under an adversarial response from nature, achieves the lowest log loss), they will choose the maximum entropy distribution out of all of their available options: this in fact mirrors the behavior of current state-of-the-art marginal-based methods~\cite{mckenna2019graphical,mckenna2022aim}.
Second, given an optimal modeler, the log loss is the same no matter what distribution nature chooses, and in fact the log loss is equal to the entropy of the modeler's guess.
This second insight will prove useful in the proof for our central result, which we now state.

\begin{theorem}
\label{thm:main}
    $KL(Q_M^*||P) = H(Q_M^*) - H(P)$ for any $P\in\Pi_M$.
\end{theorem}

\begin{proof}
By~\cref{lemma:1}, we know that $L(P, Q_M^*) = H(Q_M^*)$.
The KL divergence is thus

$$
KL(Q_M^*||P) = L(P,Q_M^*) - H(P) = H(Q_M^*) - H(P),
$$

proving the theorem.
\end{proof}

\cref{thm:main} gives a closed-form expression for the KL divergence between the modeler's estimated distribution and a hypothetical real distribution, expressed in terms of the entropy of the estimated distribution and the entropy of the real distribution.
\cref{thm:main} thus highlights when we should expect the estimate from a marginal-based method to be good or bad: is the real distribution high-entropy or low-entropy?
If the real distribution is high-entropy, then we should expect a marginal-based approach to achieve low KL divergence; on the other hand, if the real distribution is low-entropy, we should expect a marginal-based approach to provide a guess that is too ``smoothed out'', and far away from ground truth.
Furthermore, it grounds the intuition that the less information that is provided to the model ($M$ is small), the larger the KL divergence.
If less information is provided, then $H(Q_M^*)$ will be larger, and thus the KL divergence will also be larger.
We continue discussion of~\cref{thm:main} in~\cref{paragraph:dyck,subsubsec:low_privacy}.

With~\cref{thm:main} in hand, we can now quickly derive an impossibility result that makes the key limitation of marginal-based approaches precise.

\begin{corollary}
\label{corollary:1}
When $M<K$, $KL(Q_M^*||P_M) > 0$.
\end{corollary}

\begin{proof}
When a marginal-based method is not given the full joint $\hat{P}(X_1,\ldots,X_K)$, then $|\Pi_M|\geq2$.
Since there is only one maximum entropy distribution $Q_M^*$ (\cref{lemma:1}), and since nature breaks ties by choosing the distribution with the smallest entropy (\cref{par:preliminaries,lemma:1}), then that means $H(Q_M^*) > H(P_M)$, which implies that $KL(Q_M^*||P_M) > 0$.
\end{proof}

Since marginal-based approaches learn distributions via incomplete information (i.e., low-order marginals), then even in the most ideal of settings (infinite data, privacy budget), they cannot distinguish between one of potentially many choices of the true distribution: this leads to an error that does not asymptotically vanish with more data or privacy budget.

In the next two sub-sections, we present two empirical settings that demonstrate this exact limitation of marginal-based approaches.

\subsubsection{Selected Datasets}

In this section, we present two datasets with complex joint interactions---\textit{Dyck-$k$} and WikiText-103~\cite{merity2016pointer}---in which DP-TBART is able to achieve better performance than AIM under tight privacy guarantees.
Both selected datasets are  ideal testbeds for our approach because they serve as \textit{worst-case} scenarios for the marginal-based approach.
Seen through the lens of~\cref{thm:main}, measuring only low-order marginals leads to an estimated distribution with significantly higher entropy than the ground truth distribution (which contains much redundancy): this leads to a large KL divergence.
Intuitively speaking, it is difficult for the marginal-based method to guess the higher-order correlations accurately, so we would expect its performance to be quite far from optimal.

\paragraph{Dyck Dataset}
\label{paragraph:dyck}

In this section, we compare DP-TBART to AIM on \textit{Dyck-$k$}, a family of datasets we define and artificially construct.\footnote{Note that in NLP, Dyck-$k$ may refer to the language of nested brackets of $k$ types.
We consider only languages with one bracket type.}
Informally speaking, \textit{Dyck-$k$} consists of all length $k$ strings that belong to the Dyck formal language: the set of all strings with well-matched parentheses.
Formally, we can define \textit{Dyck} as the set of all strings generated by the following context-free grammar:
\begin{align*}
X\mapsto &| \hspace{2pt} ( \hspace{2pt} X \hspace{2pt} ) \hspace{2pt} X \\
         &| \hspace{2pt} \epsilon,
\end{align*}

where $\epsilon$ is the empty string.

\textit{Dyck-$k$} is thus the set of all strings in \textit{Dyck} with length $k$.
To give a concrete example, ``$((()))$'' and ``$(()())$'' belong to \textit{Dyck-$6$}.
On the other hand, ``$)(())($'' does \textit{not} belong to \textit{Dyck-$6$}; it has length $6$ but does not have well-matched parentheses.

\begin{table}
  \centering
  \begin{tabular}{lc}
    \toprule
    Method & \% Valid $\uparrow$  \\
    \midrule
    AIM & 0.517  \\
    DP-TBART & \textbf{0.804} \\
    \bottomrule
  \end{tabular}
  \vspace{5pt}
    \caption{We evaluate DP-TBART against AIM on \textit{Dyck-$20$} and report its performance with validity as the metric.
    }
    \label{tab:dyck_table}
\end{table}

To study the performance of our approach, we construct a comprehensive dataset consisting of all strings in \textit{Dyck-$20$}.
Then, we evaluate each method by fitting it to the real dataset in a differentially private manner ($\epsilon = 1$ and $\delta = 1e-9$), sample from the fitted model, and measure its performance by counting the frequency with which it samples syntactically correct strings.

In~\cref{tab:dyck_table}, we show the performance of our approach along with AIM.
We find that our approach outperforms the competing marginal-based approach by a wide margin, achieving 80.4\% accuracy in sampling syntactically correct strings.

While \textit{Dyck-$k$} is an artifically-constructed dataset, it serves as a concrete empirical example of the limitations of marginal-based methods: even in the best of circumstances, they may be unable to capture the full complexity of a given joint distribution.

\paragraph{WikiText-103}

We also compare DP-TBART to AIM on WikiText-103~\cite{merity2016pointer}, a real-world text dataset of Wikipedia articles.
We split WikiText-103 into contiguous 50-character substrings and treat each substring as a row in our tabular dataset, where each column represents a distinct character index.

\begin{table*}[t]
  \centering
  \begin{tabular}{lc}
    \toprule
    Method & GPT-2 Perplexity $\downarrow$ \\
    \midrule
    AIM & 3948  \\
    DP-TBART & 2945 \\
    DP-TBART (no privacy) & 2334 \\
    \bottomrule
  \end{tabular}
    \caption{
    We compare perplexities of outputs from DP-TBART with outputs from AIM on the WikiText-103 dataset as evaluated by GPT-2.
    }
    \label{tab:wikitext_table}
\end{table*}

We evaluate each model by fitting to the processed dataset, sampling, and then taking the average of the perplexity of the sampled rows computed with GPT-2 ~\cite{radford2019language}.
In~\cref{tab:wikitext_table}, we show the performance of our approach, along with AIM.
We find that our approach substantially outperforms AIM, achieving 25\% lower perplexity, which suggests that the full joint distribution of characters is better captured by our approach than by AIM.

\subsubsection{Low Privacy Setting}
\label{subsubsec:low_privacy}

As pointed out in~\cref{corollary:1}, even under the most generous assumptions (i.e., perfect marginal measurements, infinite data), marginal-based methods can fail to accurately capture the target distribution.
On the other hand, deep learning-based methods have no such intrinsic limitation.
In this section, we analyze what happens in practice when we approach this ideal noiseless setting by letting the privacy budget grow to $\epsilon=100$ or greater.

In~\cref{fig:low_privacy}, we show the performance of our approach and the leading marginal-based approach~\cite{mckenna2022aim} when varying $\epsilon$ on the Adult dataset.
We find that while both approaches benefit from the increased privacy budget, our approach consistently outperforms the marginal-based approach across values of $\epsilon$ exceeding $1000$ on the DET metric, suggesting that our approach is able to better capture higher-order correlations in the data when the privacy budget is relaxed.

\begin{figure}
    \centering
    \includegraphics[width=0.4\linewidth]{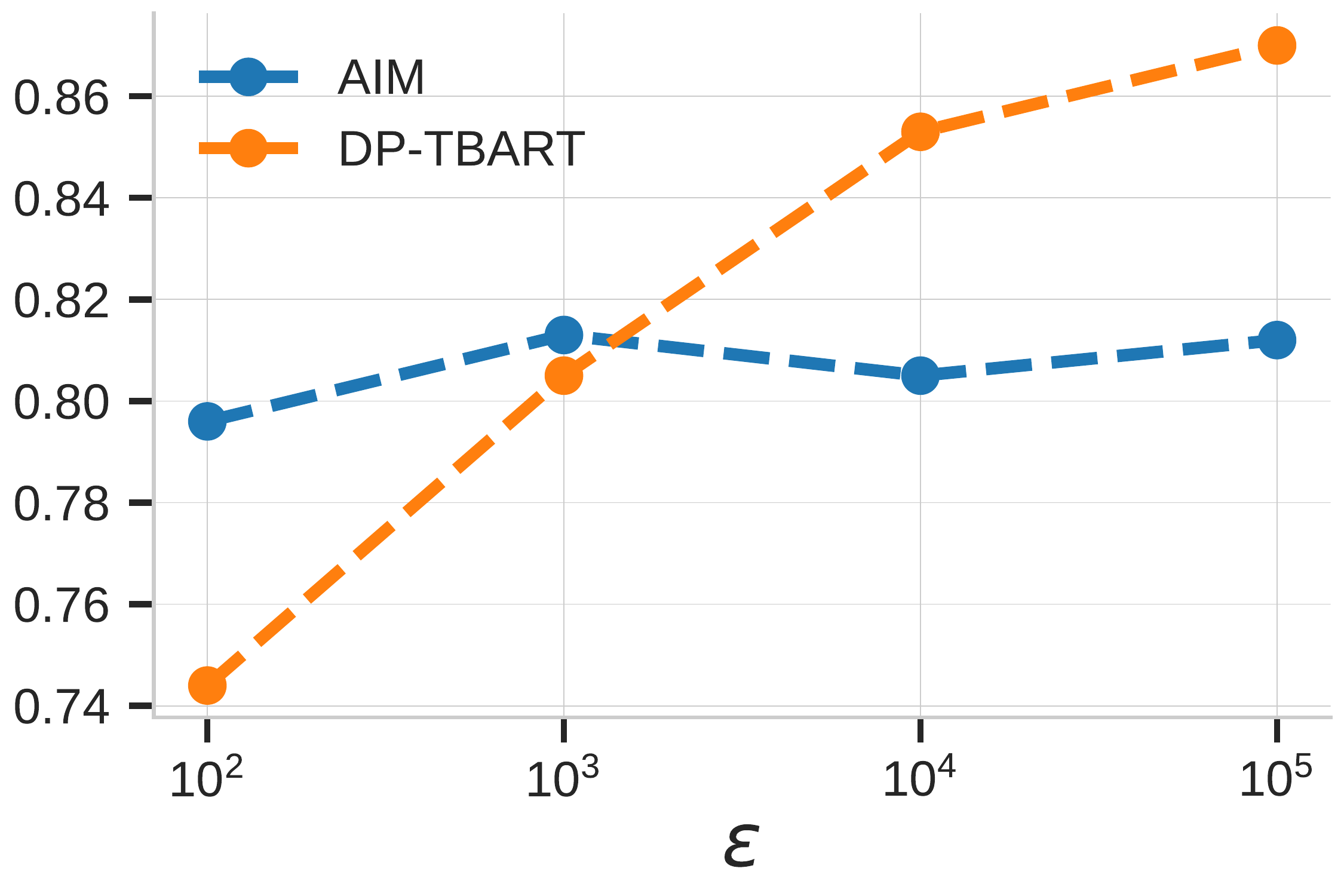}
    \caption{We show the performance (DET metric) of DP-TBART and AIM when varying $\epsilon$ on the Adult dataset.}
    \label{fig:low_privacy}
    \vspace{-10pt}
\end{figure}

\section{Discussion}

\subsection{Deep Learning for Synthetic Data}

Though synthetic data has been studied in the deep learning community, many studies overlook privacy concerns by failing to implement differential privacy, leading to potential leakage of private data~\cite{xu2019modeling,zhao2021ctab,zhao2022ctab,borisov2022language,kotelnikov2022tabddpm}.
Additionally, these studies have also often neglected to compare their methods with state-of-the-art approaches (i.e., marginal-based).

In contrast, our work is the first to evaluate a deep learning method against the state-of-the-art across several datasets under strict privacy guarantees.
We contextualize our results with further analysis, providing the first theoretical and empirical insights into the relative strengths of marginal-based methods and deep learning-based methods for synthetic tabular data generation.

\subsection{The promise of deep learning-based approaches}

Though DP-TBART does not outperform AIM on standard benchmarks, our auxiliary results suggest that deep learning-based methods can fill an important niche that marginal-based methods are not able to.
We are optimistic that, through further investigation, deep learning-based approaches have the potential to surpass marginal-based methods in a broader range of scenarios or even pave the way for a hybrid method that incorporates the strengths of both approaches.

\section{Conclusion}

In this paper, we have presented an approach for differentially private synthetic tabular data generation based on deep learning.
We have shown that, in certain settings, our deep learning-based approach is able to outperform state-of-the-art marginal-based approaches.
We further provide a theoretical framework that formalizes the limitations of marginal-based methods and two empirical settings that demonstrate these limitations in action.
Finally, we discuss the potential of deep learning-based approaches for further exploration.

We believe that our work serves as a promising starting point for understanding the utility of deep learning-based approaches for generating synthetic data under privacy guarantees, and hope that our results may inspire further exploration in this area.

\subsection{Limitations \& Broader Impact}

Though our theoretical framework allows us to understand the limitations of different modeling strategies, it also rests on assumptions that are not satisfied in practice, e.g., infinite data and noiseless marginal measurements.
Despite these unrealistic assumptions, we believe that the framework is intuitive enough to provide insight into practical settings (and confirm this with experiments, see~\cref{subsubsec:low_privacy}).
As for DP-TBART itself, it currently handles only discrete (or discretized) data, underperforms relative to the state-of-the-art marginal-based method, and requires GPU hardware to run efficiently.
Despite row-level differential privacy guarantees, outputs from DP-TBART could still reflect underlying biases in the training data, which could lead to harm.

\bibliography{main}

%%%%%%%%%%%%%%%%%%%%%%%%%%%%%%%%%%%%%%%%%%%%%%%%%%%%%%%%%%%%

%%%%%%%%%%%%%%%%%%%%%%%%%%%%%%%%%%%%%%%%%%%%%%%%%%%%%%%%%%%%%%%%%%%%%%%%%%%%%%%
%%%%%%%%%%%%%%%%%%%%%%%%%%%%%%%%%%%%%%%%%%%%%%%%%%%%%%%%%%%%%%%%%%%%%%%%%%%%%%%
% APPENDIX
%%%%%%%%%%%%%%%%%%%%%%%%%%%%%%%%%%%%%%%%%%%%%%%%%%%%%%%%%%%%%%%%%%%%%%%%%%%%%%%
%%%%%%%%%%%%%%%%%%%%%%%%%%%%%%%%%%%%%%%%%%%%%%%%%%%%%%%%%%%%%%%%%%%%%%%%%%%%%%%
\newpage
\appendix
\onecolumn

\section{Additional possible baselines}

Despite the existence of CTAB-GAN+~\cite{zhao2022ctab}, a more recent deep learning approach that claims differential privacy guarantees, we did not evaluate against this model.

The original paper does not contain a complete proof that it satisfies differential privacy, and in fact there is reason to believe that it does not.
The pre-processing pipeline is deterministic, even for continuous and mixed-type columns, for which it learns a Gaussian mixture.
Furthermore, the privacy budget calculation remains the same even when both CTAB-GAN+ performs a modified batch sampling scheme that upweights infrequent rows (training-by-sampling) \textit{and} takes gradient steps w.r.t. information loss, which is a deterministic function of real samples.

Note that while our paper does not contain an explicit proof of privacy, we have introduced no new mechanisms (such as training-by-sampling or modified batch sampling above), and thus the proof in~\citet{abadi2016deep} remains valid for our algorithm.

\section{Dataset sources and provenance}

For ease of reproducibility and building off of our work, we will describe the provenance for our datasets: where they came from and what pre-processing steps we performed (if any).

\subsection{Adult}

\begin{itemize}
    \item Link: \href{https://github.com/ryan112358/private-pgm/blob/master/data/adult.csv}{https://github.com/ryan112358/private-pgm/blob/master/data/adult.csv}
    \item Provenance: We downloaded the above CSV file without any further pre-processing.
    \item Description: This dataset is originally sourced from the \href{https://archive.ics.uci.edu/ml/datasets/adult}{UCI Machine Learning repository}~\cite{Dua:2019}, but has been discretized in the process.
    The dataset contains personal attributes of U.S. residents, sampled in a stratified manner. The data was extracted from the 1994 U.S. Census.
    \item Licensing: Apache License 2.0
\end{itemize}

\subsection{King}

\begin{itemize}
    \item Link: \href{https://github.com/Team-TUD/CTAB-GAN-Plus/blob/main/Real\_Datasets/king.csv}{https://github.com/Team-TUD/CTAB-GAN-Plus/blob/main/Real\_Datasets/king.csv}
    \item Provenance: We downloaded the above CSV file without any further pre-processing.
    \item Description: The King dataset was originally obtained from \href{https://www.kaggle.com/datasets/harlfoxem/housesalesprediction}{this Kaggle dataset} and contains house attributes (sqft, \# of bathrooms, etc.) labeled with their sale price between 2014 and 2015 for King County, Washington.
    \item Licensing: CC0: Public Domain
\end{itemize}

\subsection{Insurance}

\begin{itemize}
    \item Link: \href{https://www.kaggle.com/datasets/mirichoi0218/insurance}{https://www.kaggle.com/datasets/mirichoi0218/insurance}
    \item Provenance: We downloaded `insurance.csv' without any further pre-processing
    \item Description: The Insurance dataset contains personal attributes (age, etc.) and U.S. medical insurance costs.
    \item Licensing: Database Contents License
\end{itemize}

\subsection{Loan}

\begin{itemize}
    \item Link: \href{https://www.kaggle.com/datasets/itsmesunil/bank-loan-modelling}{https://www.kaggle.com/datasets/itsmesunil/bank-loan-modelling}
    \item Provenance: We downloaded `Bank\_Personal\_Loan\_Modelling.xlsx' and processed it into a CSV.
    \item Description: The Loan dataset contains personal attributes from bank customers and a label (whether they converted from a depositor to loan customer).
    \item Licensing: CC0: Public Domain
\end{itemize}

\subsection{Credit}

\begin{itemize}
    \item Link: \href{https://www.kaggle.com/mlg-ulb/creditcardfraud}{https://www.kaggle.com/mlg-ulb/creditcardfraud}
    \item Provenance: We downloaded the `creditcard.csv' with a Kaggle account with no additional pre-processing.
    \item Description: The Credit dataset contains 30 numerical attributes describing credit card transactions along with a label for whether they were fraudulent.
    \item Licensing: Database Contents License
\end{itemize}

\subsection{Bank}

\begin{itemize}
    \item Link: \href{https://archive.ics.uci.edu/ml/datasets/Bank+Marketing}{https://archive.ics.uci.edu/ml/datasets/Bank+Marketing}
    \item Provenance: We downloaded `bank-full.csv' from the \href{https://archive.ics.uci.edu/ml/datasets/adult}{UCI Machine Learning repository}~\cite{Dua:2019} with no additional pre-processing.
    \item Description: The Bank dataset~\cite{moro2014data} is derived from marketing campaigns from a bank. It contains personal attributes about a client and whether the client has subscribed to a product (bank term deposit).
    \item Licensing: Creative Commons Attribution 4.0 International
\end{itemize}

\subsection{Census}

\begin{itemize}
    \item Link: \href{https://archive.ics.uci.edu/ml/datasets/Census+Income}{https://archive.ics.uci.edu/ml/datasets/Census+Income}
    \item Provenance: We downloaded `census.tar.gz' from the \href{https://archive.ics.uci.edu/ml/datasets/adult}{UCI Machine Learning repository}~\cite{Dua:2019} and extracted it.
    Then, we combined census-income.data and census-income.test by concatenating both files together.
    \item Description: The Census dataset contains demographic and employment-related personal attributes from the U.S. Census Bureau, where each row corresponds to a stratified sample from the U.S. population.
    \item Licensing: Creative Commons Attribution 4.0 International
\end{itemize}

\subsection{Car}

\begin{itemize}
    \item Link: \href{https://archive.ics.uci.edu/ml/datasets/Car+Evaluation}{https://archive.ics.uci.edu/ml/datasets/Car+Evaluation}
    \item Provenance: We downloaded `car.data' from the \href{https://archive.ics.uci.edu/ml/datasets/adult}{UCI Machine Learning repository}~\cite{Dua:2019} and performed no further pre-processing.
    \item Description: The Car dataset has only categorical columns and each row describes attributes of a different car model.
    \item Licensing: Creative Commons Attribution 4.0 International
\end{itemize}

\subsection{Mushroom}

\begin{itemize}
    \item Link: \href{https://archive.ics.uci.edu/ml/datasets/Mushroom}{https://archive.ics.uci.edu/ml/datasets/Mushroom}
    \item Provenance: We downloaded `agaricus-lepiota.data' from the \href{https://archive.ics.uci.edu/ml/datasets/adult}{UCI Machine Learning repository}~\cite{Dua:2019} and performed no further pre-processing.
    \item Description: The Mushroom dataset contains descriptions of hypothetical samples corresponding to 23 species of mushrooms.
    \item Licensing: Creative Commons Attribution 4.0 International
\end{itemize}

\subsection{Poker Hands}

\begin{itemize}
    \item Link: \href{https://archive.ics.uci.edu/ml/datasets/Poker+Hand}{https://archive.ics.uci.edu/ml/datasets/Poker+Hand}
    \item Provenance: We downloaded `poker-hand-training-true.data' from the \href{https://archive.ics.uci.edu/ml/datasets/adult}{UCI Machine Learning repository}~\cite{Dua:2019} and performed no further pre-processing.
    \item Description: The Poker Hands dataset contains hands of five playing cards and their corresponding class (in Poker).
    \item Licensing: Creative Commons Attribution 4.0 International
\end{itemize}

\begin{figure*}[t]
    \centering
    \includegraphics[width=0.9\linewidth]{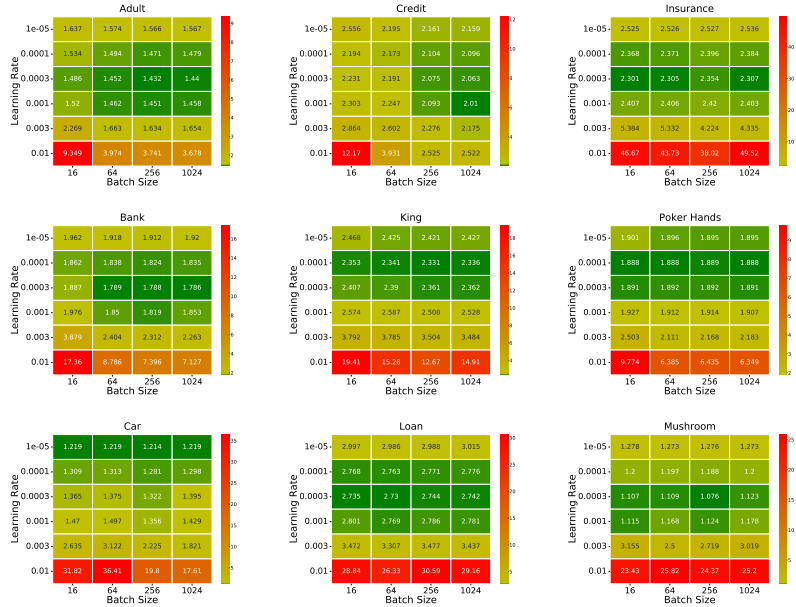}
    \caption{We train models with varying learning rates and batch sizes and evaluate the validation loss of each trained model.}
    \label{fig:lr_batch_figure}
    \vspace{-10pt}
\end{figure*}

\section{Hyperparameter Study}

Given observations from past work from image and text domains that networks trained with differential privacy are sensitive to hyperparameters~\cite{papernot2019making,li2021large,de2022unlocking}, we investigate whether the same phenomena occur in the tabular domain as well.

\paragraph{Learning Rate \& Batch Size}

Our learning rate and batch size sweeps reveal behaviors that mirror those in networks trained on images~\cite{kurakin2022toward,de2022unlocking} or text~\cite{li2021large}: performance is highly sensitive to learning rate and batch size, and increasing batch size almost never hurts.
See~\cref{fig:lr_batch_figure} for a visualization of our results.

\paragraph{Clipping Norm}

We further study the behavior of the network under various clipping norms and precisely explain the underlying mechanism that causes the observed behavior.
As shown in~\cref{fig:clipnorm}, we find that the network behaves similarly to what has been empirically observed in images~\cite{kurakin2022toward} and text~\cite{li2021large}: decreasing the clipping norm until all gradients are clipped results in the best performance.
Unlike prior work, however, we continue to evaluate the model at smaller and smaller values of $C$, and observe that the performance of the network actually stays constant and then begins to \textit{rapidly degrade}.

While the plateauing in performance is explained by (1) the fact that Adam is gradient scale-invariant (i.e., scaling each gradient $g$ by some scalar does not impact parameter updates) and (2) the fact that all gradients are clipped below $C=1$, this does not adequately explain the rapid degradation of performance when $C$ reaches $10^{-5}$.

While it may be tempting to conclude that such significant degradation of performance is due to floating point imprecision, it turns out that the answer is much more subtle.
Adam computes the effective step $\Delta_t = \alpha \cdot \hat{m}_t / (\sqrt{\hat{v}_t} + \epsilon)$ where $\epsilon=1e-8$ in practice and $\hat{m}_t$ and $\sqrt{\hat{v}_t}$ are numerical estimates proportional to clipped gradient magnitude.
When the gradient magnitudes reach the same order of magnitude as $\epsilon$, the denominator of the fraction becomes dominated by $\epsilon$ and training begins to falter, as seen in~\cref{fig:clipnorm}.

\begin{figure}
    \centering
    \includegraphics[width=0.4\linewidth]{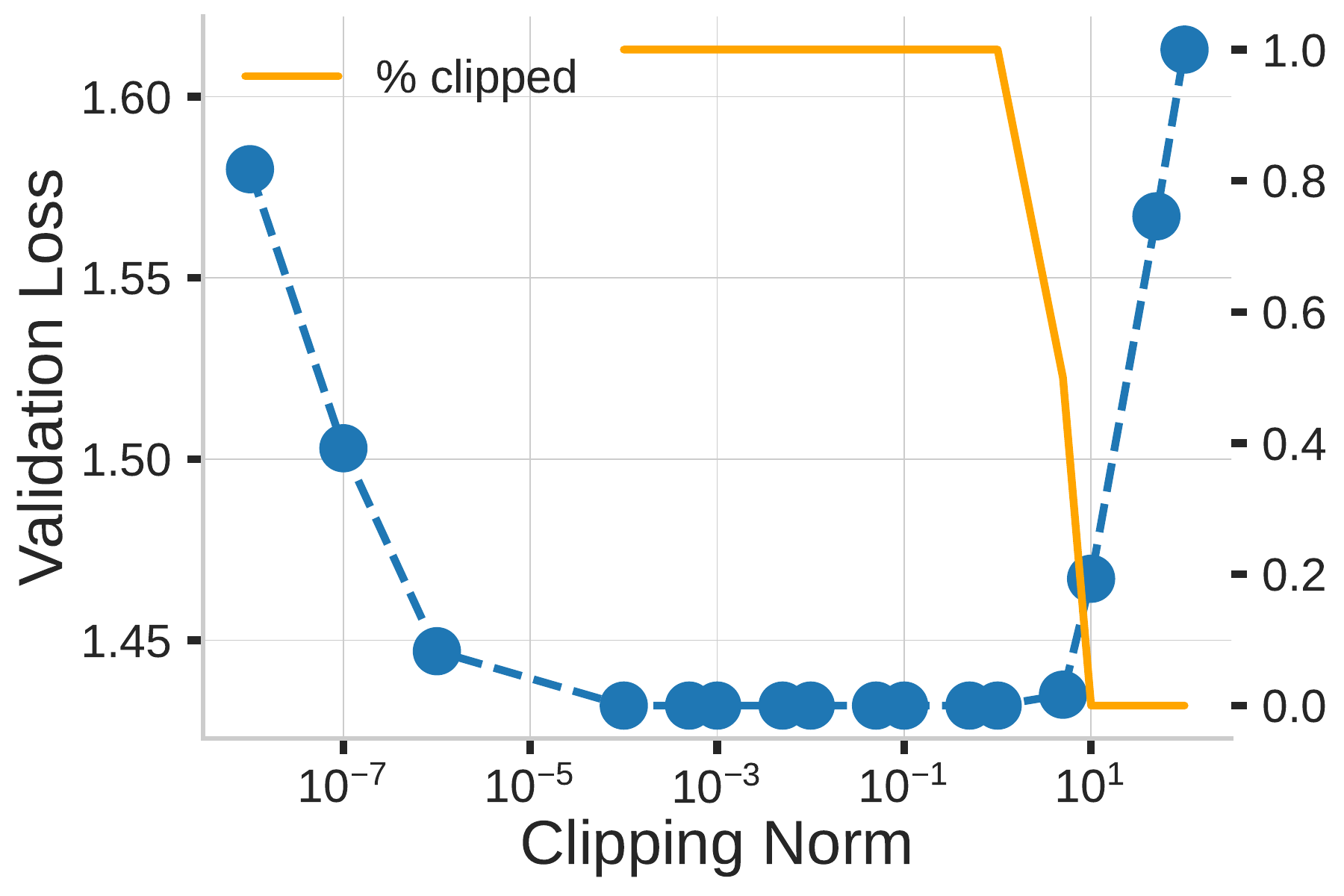}
    \caption{We show the impact of clipping norm on network performance on the Adult dataset.}
    \label{fig:clipnorm}
    \vspace{-10pt}
\end{figure}

\section{Software}

For model implementation and training, we use Pytorch~\cite{PaszkePyTorchAnImperative2019} (BSD-3) and Lightning~\cite{FalconPyTorchLightning2019} (Apache License 2.0).
For experiment tracking, we use Aim~\cite{ArakelyanAim2020} (Apache License 2.0).
For our transformer model, we adapted \href{https://huggingface.co/distilgpt2}{DistilGPT-2}~\cite{sanh2019distilbert} from Hugging Face Transformers~\cite{WolfTransformersStateoftheArtNatural2020} (Apache License 2.0).

For our DP-SGD implementation, we use \href{https://github.com/lxuechen/private-transformers}{private-transformers} (Apache License 2.0).

\section{Poisson sampling}

While some previous works sample batches non-privately by shuffling the training dataset and picking uniform batches~\cite{li2021large,tramer2020differentially}, we ensure end-to-end differential privacy guarantees by implementing Poisson sampling.
Since the Binomial distribution $B(n,p)$ concentrates tightly around its mean when $n$ is large and $p$ is small, we observe that out-of-memory issues do not occur in practice, and micro-batching was not required.

\section{Column Ordering}

One might observe that an autoregressive approach introduces an additional high-dimensional hyperparameter: column ordering.
In our experiments, we attempted several strategies (ordering columns from ones with largest cardinalities to smallest, smallest to largest, and toposorting a graph learned via structure learning), but observed no consistent impact on performance.

\end{document}